\documentclass{article}
\usepackage[latin9]{inputenc}
\usepackage{color}
\usepackage{float}
\usepackage{booktabs}
\usepackage{bm}
\usepackage{amsthm}
\usepackage{amsmath}
\usepackage{amssymb}
\usepackage{graphicx}
\usepackage[small]{caption}

\makeatletter

\providecommand{\tabularnewline}{\\}
\floatstyle{ruled}
\newfloat{algorithm}{tbp}{loa}
\providecommand{\algorithmname}{Algorithm}
\floatname{algorithm}{\protect\algorithmname}

  \theoremstyle{definition}
  \newtheorem{defn}{\protect\definitionname}
  \theoremstyle{definition}
  \newtheorem*{example*}{\protect\examplename}
  \theoremstyle{plain}
  \newtheorem{lem}{\protect\lemmaname}
\theoremstyle{plain}
\newtheorem{thm}{\protect\theoremname}
  \theoremstyle{plain}
  \newtheorem{cor}{\protect\corollaryname}

\usepackage{ijcai17}

\usepackage{times}
\usepackage{algorithm,algpseudocode}

\title{Count-Based Exploration in Feature Space for Reinforcement Learning\thanks{This work was supported in part by ARC DP150104590.}}

\author{Jarryd Martin, Suraj Narayanan S., Tom Everitt, Marcus Hutter\\
Research School of Computer Science, Australian National University, Canberra\\
jarrydmartinx@gmail.com, surajx@gmail.com, tom.everitt@anu.edu.au, marcus.hutter@anu.edu.au}

\makeatother

  \providecommand{\definitionname}{Definition}
  \providecommand{\examplename}{Example}
  \providecommand{\lemmaname}{Lemma}
\providecommand{\corollaryname}{Corollary}
\providecommand{\theoremname}{Theorem}

\begin{document}
\maketitle
\begin{abstract}
We introduce a new count-based optimistic exploration algorithm for
reinforcement learning (RL) that is feasible in environments with
high-dimensional state-action spaces. The success of RL algorithms
in these domains depends crucially on generalisation from limited
training experience. Function approximation techniques enable RL agents
to generalise in order to estimate the \emph{value} of unvisited states,
but at present few methods enable generalisation regarding\emph{ uncertainty}.
This has prevented the combination of scalable RL algorithms with
efficient \emph{exploration} strategies that drive the agent to reduce
its uncertainty. We present a new method for computing a generalised
state visit-count, which allows the agent to estimate the uncertainty
associated with any state. Our\emph{ $\phi$-pseudocount} achieves
generalisation by exploiting the same feature representation of the
state space that is used for value function approximation. States
that have less frequently observed features are deemed more uncertain.
The \emph{$\phi$-Exploration-Bonus algorithm} rewards the agent for
exploring in feature space rather than in the untransformed state
space. The method is simpler and less computationally expensive than
some previous proposals, and achieves near state-of-the-art results
on high-dimensional RL benchmarks.
\end{abstract}

\section{Introduction}

Reinforcement learning (RL) methods have recently enjoyed widely publicised
success in domains that once seemed far beyond their reach \cite{Mnih2015}.
Much of this progress is due to the application of modern function
approximation techniques to the problem of policy evaluation for Markov
Decision Processes (MDPs) \cite{Sutton1998}. These techniques address
a key shortcoming of tabular MDP solution methods: their inability
to generalise what is learnt from one context to another. This sort
of generalisation is crucial if the state-action space of the MDP
is large, because the agent typically only visits a small subset of
that space during training.

Comparatively little progress has been made on the problem of efficient
\emph{exploration} in large domains. Even algorithms that use sophisticated
nonlinear methods for policy evaluation tend to use very old, inefficient
exploration techniques, such as the $\epsilon$-greedy strategy \cite{HGS:2016doubleQ,MBMG+2016DQN,DBLP:journals/corr/NairSBAFMPSBPLM15}.
There are more efficient tabular \emph{count-based }exploration algorithms
for finite MDPs, which drive the agent to reduce its uncertainty by
visiting states that have low visit-counts \cite{Strehl2008}. However,
these algorithms are often ineffective in MDPs with high-dimensional
state-action spaces, because most states are never visited during
training, and the visit-count remains at zero nearly everywhere.

Count-based exploration algorithms have only very recently been successfully
adapted for these large problems \cite{Bellemare2016,Tang2016exploration}.
Just as function approximation techniques achieve generalisation across
the state space regarding value, these algorithms achieve generalisation
regarding \emph{uncertainty}. The breakthrough has been the development
of \emph{generalised state visit-counts, }which are larger for states
that are more \emph{similar} to visited states, and which can be nonzero
for unvisited states. The key challenge is to compute an appropriate
similarity measure in an efficient way, such that these exploration
methods can be combined with scalable RL algorithms. It soon becomes
infeasible, for example, to do so by storing the entire history of
visited states and comparing each new state to those in the history.
The most promising proposals instead compute generalised counts from
a compressed representation of the history of visited states -- for
example, by constructing a \emph{visit-density model} over the state
space and deriving a ``pseudocount'' \cite{Bellemare2016,DBLP:journals/corr/OstrovskiBOM17},
or by using locality-sensitive hashing to cluster states and counting
the occurrences in each cluster \cite{Tang2016exploration}.

This paper presents a new count-based exploration algorithm that is
feasible in environments with large state-action spaces. It can be
combined with any value-based RL algorithm that uses linear function
approximation (LFA). Our principal contribution is a \emph{new method
for computing generalised visit-counts}. Following \cite{Bellemare2016},
we construct a visit-density model in order to measure the similarity
between states. Our approach departs from theirs in that we do \emph{not}
construct our density model over\emph{ }the raw state space. Instead,
we exploit the feature map that is used for value function approximation,
and construct a density model over the transformed \emph{feature space}.
This model assigns higher probability to state feature vectors that
\emph{share features} with visited states. Generalised visit-counts
are then computed from these probabilities; states with frequently
observed features are assigned higher counts. These counts serve as
a measure of the uncertainty associated with a state. \emph{Exploration
bonuses} are then computed from these counts in order to encourage
the agent to visit regions of the state-space with less familiar features.

Our density model can be trivially derived from \emph{any }feature
map used for LFA, regardless of the application domain, and requires
little or no additional design. In contrast to existing algorithms,
there is no need to perform a special dimensionality reduction of
the state space in order to compute our generalised visit-counts.
Our method uses the same lower-dimensional feature representation
to estimate value \emph{and }to estimate uncertainty. This makes it
simpler to implement and less computationally expensive than some
existing proposals. Our evaluation demonstrates that this simple approach
achieves near state-of-the-art performance on high-dimensional RL
benchmarks.

\section{Background and Related Work}

\subsection{Reinforcement Learning}

The reinforcement learning (RL) problem formalises the task of learning
from interaction to achieve a goal \cite{Sutton1998}. It is usually
formulated as an MDP $\langle\mathcal{S},\mathcal{A},\mathcal{P},\mathcal{R},\gamma\rangle$,
where $\mathcal{S}$ is the set of states of the environment, $\mathcal{A}$
is the set of available actions, $\mathcal{P}:(S\times\mathcal{A})\times S\to[0,1]$
is the state transition distribution, $\mathcal{R}:(\mathcal{S}\times\mathcal{A})\times\mathcal{S}\to\mathbb{R}$
is the reward function, and $\gamma$ is the discount factor. The
agent is formally a \emph{policy }$\pi:\mathcal{S}\to\mathcal{A}$
that maps a state to an action. At timestep $t$, the agent is in
a \emph{state} $s_{t}\in\mathcal{S}$, receives a reward $r_{t}$,
and takes an \emph{action} $a_{t}\in\mathcal{A}$. We seek a policy
$\pi$ that maximises the \emph{expected }sum of future rewards, or
\emph{value}. The\emph{ action-value }$Q^{\pi}(s,a)$ of a state-action
pair $(s,a)$ under a policy $\mbox{\ensuremath{\pi}}$ is the expected
discounted sum of future rewards, given that the agent takes action
$a$ from state $s$, and follows $\pi$ thereafter: $Q^{\pi}(s,a)=\mathbb{E_{\pi}}\big[\sum_{k=0}^{\infty}\gamma^{k}r_{t+k+1}\mid s_{t}=s,\ a_{t}=a\big]$. 

RL methods that compute a \emph{value function }are called \emph{value-based}
methods. \emph{Tabular} methods store the value function as a table
having one entry for each state(-action).\emph{ }This representation
of the state space does not have sufficient structure to permit generalisation
based on the \emph{similarity} between states. \emph{Function approximation}
methods achieve generalisation by approximating the value function\emph{
}by a parameterised functional form. In LFA the \emph{approximate
action-value function} $\hat{Q}_{t}^{\pi}(s,a)=\bm{\theta}_{t}^{\top}\bm{\phi}(s,a)$
is a linear combination of state-action features, where $\bm{\phi}:\ \mathcal{S}\times\mathcal{A}\to\mathcal{T}\subseteq\mathbb{R}^{M}$
is an $M$-dimensional feature map and $\bm{\theta}_{t}\in\mathbb{R}^{M}$
is a parameter vector.

\subsection{Count-Based Exploration and Optimism}

Since the true transition and reward distributions $\mathcal{P}$
and $\mathcal{R}$ are unknown to the agent, it must \emph{explore}
the environment to gather more information and reduce its uncertainty.
At the same time, it must \emph{exploit} its current information to
maximise expected cumulative reward. This tradeoff between exploration
and exploitation is a fundamental problem in RL.

Many of the exploration algorithms that enjoy strong theoretical guarantees
implement the `\emph{optimism in the face of uncertainty}' (OFU) heuristic
\cite{Strehl2009}. Most are \emph{tabular }and \emph{count-based}
in that they compute \emph{exploration bonuses} from a table of state(-action)
visit counts. These bonuses are added to the estimated state/action
value. Lower counts entail higher bonuses, so the agent is effectively
optimistic about the value of less frequently visited regions of the
environment. OFU algorithms are more efficient than random strategies
like $\epsilon$-greedy because the agent avoids actions that yield
neither large rewards nor large reductions in uncertainty \cite{Osband2016a}. 

One of the best known is the UCB1 bandit algorithm, which selects
an action $a$ that maximises an \emph{upper confidence} \emph{bound
$\hat{Q}_{t}(a)+\sqrt{\frac{2\log t}{N(a)}}$, }where $\hat{Q}_{t}(a)$
is the estimated mean reward and $N(a)$ is the visit-count \cite{LR:1985bandits}.
The dependence of the bonus term on the inverse square-root of the
visit-count is justified using Chernoff bounds. In the MDP setting,
the tabular OFU algorithm most closely resembling our method is \emph{Model-Based
Interval Estimation with Exploration Bonuses} (MBIE-EB) \cite{Strehl2008}.%
\footnote{To the best of our knowledge, the first work to use exploration bonuses
in the MDP setting was the Dyna-Q+ algorithm, in which the bonus is
a function of the \emph{recency }of visits to a state, rather than
the visit-count \cite{Sutton:1990:IAL:101883.102055}.%
}\emph{ }Empirical estimates $\hat{\mathcal{P}}$ and $\hat{\mathcal{R}}$
of the transition and reward functions are maintained, and $\mathcal{\hat{R}}(s,a)$
is augmented with a bonus term $\frac{\beta}{\sqrt{N(s,a)}}$, where
$N(s,a)$ is the state-action visit-count, and $\beta\in\mathbb{R}$
is a theoretically derived constant. The Bellman optimality equation
for the augmented action-value function is $\tilde{Q}^{\pi}(s,a)=\mathcal{\hat{R}}(s,a)+\frac{\beta}{\sqrt{N(s,a)}}+\ \gamma\sum_{s'}\hat{\mathcal{P}}(s'\mid s,a)\max_{a'\in\mathcal{A}}\tilde{Q}^{\pi}(s',a')$.
Here the dependence of the bonus on the inverse square-root of the
visit-count is provably optimal \cite{Kolter2009}. This equation
can be solved using any MDP solution method.

\subsection{Exploration in Large MDPs\label{sub:BellemarePseudo}}

While tabular OFU algorithms perform well in practice on small MDPs
\cite{1374179}, their\emph{ sample complexity} becomes prohibitive
for larger problems \cite{Kakade:2003}. MBIE-EB, for example, has
a sample complexity bound of $\tilde{O}\big(\frac{\left|\mathcal{S}\right|^{2}\left|\mathcal{A}\right|}{\epsilon^{3}(1-\gamma)^{6}}\big)$.
In the high-dimensional setting -- where the agent cannot hope to
visit every state during training -- this bound offers no guarantee
that the trained agent will perform well. 

Several very recent extensions of count-based exploration methods
have produced impressive results on high-dimensional RL benchmarks.
These algorithms closely resemble MBIE-EB, but they substitute the
state-action visit-count for a generalised count which quantifies
the similarity of a state to previously visited states. Bellemare
\emph{et. al.} construct a Context Tree Switching (CTS) density model
over the state space such that higher probability is assigned to states
that are more similar to visited states \cite{Bellemare2016,veness2012context}.
A state \emph{pseudocount} is then derived from this density. A subsequent
extension of this work replaces the CTS density model with a neural
network \cite{DBLP:journals/corr/OstrovskiBOM17}. Another recent
proposal uses locality sensitive hashing (LSH) to cluster similar
states, and the number of visited states in a cluster serves as a
generalised visit-count \cite{Tang2016exploration}. As in the MBIE-EB
algorithm, these counts are used to compute exploration bonuses. These
three algorithms outperform random strategies, and are currently the
leading exploration methods in large discrete domains where exploration
is hard.

\section{Method}

Here we introduce the \emph{$\phi$-Exploration Bonus }($\phi$-EB)
algorithm, which drives the agent to visit states about which it is
uncertain. Following other optimistic count-based exploration algorithms,
we use a (generalised) state visit-count in order to estimate the
uncertainty associated with a state. A generalised count is a \emph{novelty}
measure that quantifies how dissimilar a state is from those already
visited. Measuring novelty therefore involves choosing a similarity
measure for states. Of course, states can be similar in myriad ways,
but not all of these are relevant to solving the MDP. If the solution
method used is value-based, then states should only be considered
similar if they share the features that are determinative of value.
This motivates us to construct a similarity measure that exploits
the feature representation that is used for value function approximation.
These features are explicitly designed\emph{ }to be relevant for estimating
value. If they were not, they would not permit a good approximation
to the true value function. This sets our method apart from the approaches
described in section \ref{sub:BellemarePseudo}. They measure novelty
with respect to a separate, exploration-specific representation of
the state space, one that bears no relation to the value function
or the reward structure of the MDP. We argue that measuring novelty
in feature space is a simpler and more principled approach, and hypothesise
that more efficient exploration will result.

\subsection{A Visit-Density over Feature Space\label{sub:A-Visit-Density-over}}

Our exploration method is designed for use with LFA, and measures
novelty with respect to a fixed feature representation of the state
space. The challenge is to measure novelty without computing the distance
between each new feature vector and those in the history. That approach
becomes infeasible because the cost of computing these distances grows
with the size of the history.

Our method constructs a \emph{density model over feature space }that
assigns higher probability to states that share more features with
more frequently observed states. Let $\bm{\phi}:\mathcal{S}\to\mathcal{T}\subseteq\mathbb{R}^{M}$
be the feature mapping from the state space into an $M$-dimensional
feature space $\mathcal{T}$. Let $\bm{\phi}_{t}=\bm{\phi}(s_{t})$
denote the state feature vector observed at time $t$. We denote the
sequence of observed feature vectors after $t$ timesteps by $\bm{\phi}_{1:t}\in\mathcal{T}^{t}$,
and denote the set of all finite sequences of feature vectors by $\mathcal{T}^{*}$.
Let $\bm{\phi}_{1:t}\bm{\phi}$ denote the sequence where $\bm{\phi}_{1:t}$
is followed by $\bm{\phi}$. The $i$-th element of $\bm{\phi}$ is
denoted by $\phi_{i}$, and the $i$-th element of $\bm{\phi}_{t}$
is $\phi_{t,i}$. 
\begin{defn}
[Feature Visit-Density]Let $\rho:\mathcal{T}^{*}\times\mathcal{T}\to[0,1]$
be a density model\emph{ }that maps a finite sequence of feature vectors
$\bm{\phi}_{1:t}\in\mathcal{T}^{*}$ to a probability distribution
over $\mathcal{T}$. The \emph{feature visit-density }$\rho_{t}(\bm{\phi})$
at time $t$ is the distribution over $\mathcal{T}$ that is returned
by $\rho$ after observing $\bm{\phi}_{1:t}$. 
\end{defn}
We construct our\emph{ }feature visit-density\emph{ }as a product
of independent factor distributions $\rho_{t}^{i}(\phi_{i})$ over
individual features $\phi_{i}\in\mbox{\ensuremath{\mathcal{U}\subseteq\mathbb{R}}}$:
\[
\rho_{t}(\bm{\phi})=\prod_{i=1}^{M}\rho_{t}^{i}(\phi_{i})
\]
If $\mathcal{U}$ is countable we can use a count-based estimator
for the factor models $\rho_{t}^{i}(\phi_{i})$, such as the empirical
estimator $\rho_{t}^{i}(\phi_{i})=\frac{N_{t}(\phi_{i})}{t}$, where
$N_{t}(\phi_{i})$ is the number of times $\phi_{i}$ has occurred.
In our implementation we use the Krichevsky-Trofimov (KT) estimator
$\rho_{t}^{i}(\phi_{i})=\frac{N_{t}(\phi_{i})+\frac{1}{2}}{t+1}$.

This density model induces a similarity measure on the feature space.
Loosely speaking, feature vectors that share component features are
deemed similar. This enables us to use $\rho_{t}(\bm{\phi})$ as a
novelty measure for states, by comparing the features of newly observed
states to those in the history. If $\bm{\phi}(s)$ has more novel
component features, $\rho_{t}(\bm{\phi})$ will be lower. By modelling
the features as independent, and using count-based estimators as factor
models, our method learns reasonable novelty estimates from very little
data. 
\begin{example*}
Suppose we use a 3-D binary feature map and that after 3 timesteps
the history of observed feature vectors is $\bm{\phi}_{1:3}=(0,1,0),(0,1,0),(0,1,0)$.
Let us estimate the feature visit densities of two \emph{unobserved}
feature vectors $\bm{\phi'}=(1,1,0)$, and $\bm{\phi}''=(1,0,1)$.
Using the KT estimator for the factor models, we have $\rho_{3}(\bm{\phi}')=\rho_{3}^{1}(1)\cdot\rho_{3}^{2}(1)\cdot\rho_{3}^{3}(0)=\frac{1}{8}\cdot\frac{7}{8}\cdot\frac{7}{8}\approx0.1$,
and $\rho_{3}(\bm{\phi}'')=\rho_{3}^{1}(1)\cdot\rho_{3}^{2}(0)\cdot\rho_{3}^{3}(1)=(\frac{1}{8})^{3}\approx0.002$.
Note that $\rho_{3}(\bm{\phi}')>\rho_{3}(\bm{\phi}'')$ because the
component features of $\bm{\phi}'$ are more similar to those in the
history. As desired, our novelty measure generalises across the state
space.
\end{example*}

\subsection{The $\phi$-pseudocount}

Here we adopt a recently proposed method for computing generalised
visit-counts from density models \cite{Bellemare2016,DBLP:journals/corr/OstrovskiBOM17}.
By analogy with these pseudocounts, we derive two \emph{$\phi$-pseudocounts}
from our feature visit-density. 
\begin{defn}
[$\phi$-pseudocount and Naive $\phi$-pseudocount]\label{def:pseudocount}
Let $\rho_{t}(\bm{\phi})$ be the feature visit-density after observing
$\bm{\phi}_{1:t}$. Let $\rho'_{t}(\bm{\phi})$ denote the same density
model after $\bm{\phi}_{1:t}\bm{\phi}$ has been observed.\end{defn}
\begin{itemize}
\item The \emph{naive $\phi$-pseudocount} \emph{$\tilde{N}_{t}^{\phi}(s)$}
for a state $s\in\mathcal{S}$ at time $t$ is
\[
\tilde{N}_{t}^{\phi}(s)=t\cdot\rho_{t}(\bm{\phi}(s))
\]

\item The \emph{$\phi$-pseudocount }$\hat{N}_{t}^{\phi}(s)$\emph{ }for
a state $s\in\mathcal{S}$ at time $t$ is 
\[
\hat{N}_{t}^{\phi}(s)=\frac{\rho_{t}(\bm{\phi}(s))(1-\rho'_{t}(\bm{\phi}(s)))}{\rho'_{t}(\bm{\phi}(s))-\rho{}_{t}(\bm{\phi}(s))}
\]

\end{itemize}
Empirically, $\hat{N}_{t}^{\phi}(s)$ is usually larger than \emph{$\tilde{N}_{t}^{\phi}(s)$}
and leads to better performance.%
\footnote{The expression for $\hat{N}_{t}^{\phi}(s)$ is derived by letting
it depend on an implicit \emph{total pseudocount }$\hat{n}$ that
can be much larger than $t$, and assuming $\rho_{t}(\bm{\phi})=\frac{\hat{N}_{t}^{\phi}(s)}{\hat{n}\cdot}$,
and $\rho'_{t}(\bm{\phi})=\frac{\hat{N}_{t}^{\phi}(s)+1}{\hat{n}+1}$
\cite{Bellemare2016}.%
}

\subsection{Reinforcement Learning with $\phi$-EB}

\begin{algorithm}
\begin{algorithmic}
\Require{$\beta$, $t_{\text{end}}$}
\While{$t<t_{\text{end}}$}
\State{Observe $\bm{\phi}(s)$, $r_t$}
\State{Compute $\rho_t(\bm{\phi})=\prod_{i}^{M}{\rho^{i}_{t}(\phi_i)}$}
\For{i in \{1,\ldots,M\}}
\State{Update $\rho_{t+1}^i$ with observed $\phi_i$}
\EndFor
\State{Compute $\rho_{t+1}(\bm{\phi})=\prod_{i}^{M}{\rho^{i}_{t+1}(\phi_i)}$}
\State{Compute $\hat{N}_t^{\phi}(s)=\frac{\rho_{t}(\bm{\phi})(1-\rho_{t+1}(\bm{\phi}))}{\rho_{t+1}(\bm{\phi})-\rho_{t}(\bm{\phi})}$}
\State{Compute $\mathcal{R}_t^{\phi}(s,a)=\frac{\beta}{\sqrt{\hat{N}_t^{\phi}(s)}}$}
\State{Set $r_t^{+}=r_t+\mathcal{R}_t^{\phi}(s,a)$}
\State{Pass $\bm{\phi}(s)$, $r_t^{+}$ to RL algorithm to update $\bm{\theta}_t$}
\EndWhile
\State{}
\Return{$\bm{\theta}_{t_{\text{end}}}$}
\end{algorithmic}

\caption{Reinforcement Learning with LFA and $\phi$-EB.\label{alg:Reinforcement-learning-with}}
\end{algorithm}
\label{alg:RLPHIEB}Following traditional count-based exploration
algorithms, we drive optimistic exploration by computing a bonus from
the $\phi$-pseudocount.
\begin{defn}
[$\phi$-Exploration Bonus] Let $\beta\in\mathbb{R}$ be a free parameter.
The \emph{$\phi$-exploration bonus }for a state-action pair $(s,a)\in\mbox{\ensuremath{\mathcal{S}}}\times\mathcal{A}$
at time $t$ is
\[
\mathcal{R}_{t}^{\phi}(s,a)=\frac{\beta}{\sqrt{\hat{N}_{t}^{\phi}(s)}}
\]

\end{defn}
As in the MBIE-EB algorithm, this bonus is added to the reward $r_{t}$.
The agent is trained on the augmented reward $r_{t}^{+}=r_{t}+R_{t}^{\phi}(s,a)$
using any value-based RL algorithm with LFA. At each timestep our
algorithm performs updates for at most $M$ estimators, one for each
feature. The cost of our method is therefore independent of the size
of the state-action space, and scales only in the number of features.
If the feature vectors are sparse, we can maintain a single prototype
estimator for all the features that have not yet been observed. Under
these conditions our method scales only in the number of \emph{observed
}features.

\section{Theoretical Results}

\noindent Here we formalise the comments made in section \ref{sub:A-Visit-Density-over}
by proving a bound that relates our pseudocount to an appropriate
similarity measure. To simplify the analysis, we prove results for
the naive $\phi$-exploration bonus $\tilde{N}_{t}^{\phi}(s)$, though
we expect analogous results to hold for $\hat{N}_{t}^{\phi}(s)$ as
well. We use the empirical estimator for the factor models in the
visit-density. Since the feature set we use in our implementation
is binary, our analysis assumes $\bm{\phi}\in\{0,1\}^{M}$. We begin
by defining a similarity measure for binary feature vectors, and prove
two lemmas.
\begin{defn}[Hamming Similarity for Binary Vectors]
\noindent \label{df:l1similarity} Let $\bm{\phi},\bm{\phi'}\in\{0,1\}^{M}$
be $M$-length binary vectors. The Hamming similarity between $\bm{\phi}$
and $\bm{\phi'}$ is $\mbox{Sim}(\bm{\phi},\bm{\phi'})=1-\frac{1}{M}\left\Vert \bm{\phi}\bm{-\phi'}\right\Vert _{1}$.
\end{defn}
\noindent Note that $\mbox{Sim}(\bm{\phi},\bm{\phi'})\in[0,1]$ for
all $\bm{\phi},\bm{\phi'}\in\{0,1\}^{M}$. The Hamming similarity
is large if $\bm{\phi}$ and $\bm{\phi'}$ share features (i.e. if
the $l_{1}$-distance between them is small). We now prove a lemma
relating the joint probability of a feature vector to the sum of the
probabilities of its factors.
\begin{lem}[AM-GM Inequality and Factorised $\rho$]
\noindent \label{lm:AMGM} Let $\bm{\phi}\in\{0,1\}^{M}$, and let
$\rho_{t}(\mbox{\ensuremath{\bm{\phi}}})=\prod_{i=1}^{M}\rho_{t}^{i}(\phi_{i})$.
Then $\sqrt{\rho(\bm{\phi})}\leq\frac{1}{M}\sum_{i=1}^{M}\rho_{t}^{i}(\phi_{i})$.\end{lem}
\begin{proof}
By the inequality of arithmetic and geometric means $\sqrt{\rho(\bm{\phi})}=\sqrt{\prod_{i=1}^{M}\rho_{t}^{i}(\phi_{i})}\leq\sqrt[M]{\prod_{i=1}^{M}\rho_{t}^{i}(\phi_{i})}\leq\frac{1}{M}\sum_{i=1}^{M}\rho_{t}^{i}(\phi_{i})$ 
\end{proof}
\noindent The following lemma relates the probability of an individual
feature to its $l_{1}$-distance from previously observed values.
\begin{lem}[Feature Visit-Density and $l_{1}$-distance]
\label{lm:comp_dist} Let \textup{$\rho_{t}^{i}(\phi_{i})=\frac{1}{t}N_{t}(\phi_{i})$.}
Then for all $\phi_{i},\phi_{k,i}\in\{0,1\}$, \textup{$\rho_{t}^{i}(\phi_{i})=\frac{1}{t}\sum_{k=1}^{t}1-\left|\phi_{i}-\phi_{k,i}\right|$.}\end{lem}
\begin{proof}
\noindent Suppose $\phi_{i}=0$:
\begin{alignat*}{1}
\rho_{t}^{i}(0) & =1-\rho_{t}^{i}(1)=1-\frac{1}{t}\sum_{k=1}^{t}\phi_{k,i}\\
 & =\frac{1}{t}\sum_{k=1}^{t}1-\left|0-\phi_{k,i}\right|=\frac{1}{t}\sum_{k=1}^{t}1-\left|\phi_{i}-\phi_{k,i}\right|
\end{alignat*}
The $\phi_{i}=1$ case follows by an almost identical argument. 
\end{proof}
\noindent The following theorem and its corollary are the major results
of this section. These connect the Hamming similarity (to previously
observed feature vectors) with both the feature visit-density and
the $\phi$-pseudocount. We show that a state which shares few features
with those in the history will be assigned low probability by our
density model, and will therefore have a low $\phi$-pseudocount.
\begin{thm}
[Feature Visit-Density and Average Similarity]\label{th:rhomaxsim}

\noindent Let $s\in\mathcal{S}$ be a state with binary feature representation
$\bm{\phi}=\bm{\phi}(s)\in\{0,1\}^{M}$, and let $\rho_{t}(\bm{\phi})=\prod_{i=1}^{M}\rho_{t}^{i}(\phi_{i})$
be its feature visit-density at time $t$. Then 
\[
\rho_{t}(\bm{\phi})\leq\frac{1}{t}\sum_{k=1}^{t}\mbox{\emph{Sim}}(\bm{\phi},\bm{\phi}_{k})
\]
\end{thm}
\begin{proof}
\begin{eqnarray*}
\rho_{t}(\bm{\phi}) & \leq & \sqrt{\rho_{t}(\bm{\phi})}\\
 & \stackrel{(a)}{\leq} & \frac{1}{M}\sum_{i=1}^{M}\rho_{t}^{i}(\phi_{i})\\
 & \stackrel{(b)}{=} & \frac{1}{M}\sum_{i=1}^{M}\frac{1}{t}\sum_{k=1}^{t}\big(1-\left|\phi_{i}-\phi_{k,i}\right|\big)\\
 & = & \frac{1}{t}\sum_{k=1}^{t}\big(1-\frac{1}{M}\sum_{i=1}^{M}\left|\phi_{i}-\phi_{k,i}\right|\big)\\
 & = & \frac{1}{t}\sum_{k=1}^{t}\big(1-\frac{1}{M}\left\Vert \bm{\phi}-\bm{\phi}_{k}\right\Vert _{1}\big)\\
 & \stackrel{(c)}{=} & \frac{1}{t}\sum_{k=1}^{t}\mbox{Sim}(\bm{\phi},\bm{\phi}_{k})
\end{eqnarray*}
where (a) follows from Lemma \ref{lm:AMGM}, (b) from Lemma \ref{lm:comp_dist},
and (c) from Definition \ref{df:l1similarity}. 
\end{proof}
\noindent We immediately get a similar bound for the naive $\phi$-pseudocount
$\tilde{N}_{t}^{\phi}(s)$.
\begin{cor}
[$\phi$-pseudocount and Total Similarity]\label{cor:avgsim}\emph{
\begin{eqnarray*}
\tilde{N}_{t}^{\phi}(s) & \leq & \sum_{k=1}^{t}\mbox{Sim}(\bm{\phi},\bm{\phi}_{k})
\end{eqnarray*}
}\end{cor}
\begin{proof}
Immediate from Theorem \ref{th:rhomaxsim} and Definition \ref{def:pseudocount}. 
\end{proof}
\noindent $\tilde{N}_{t}^{\phi}(s)$ therefore captures an intuitive
relation between novelty and similarity to visited states. By visiting
a state that minimises the $\phi$-pseudocount, an agent also minimises
a lower bound on its Hamming similarity to previously visited states.
As desired, we have a novelty measure that is closely related to the
distances between states in feature space, but which obviates the
cost of computing those distances directly.

\section{Empirical Evaluation}

Our evaluation is designed to answer the following research questions:
\begin{itemize}
\item Is a novelty measure derived from the features used for LFA a good
way to generalise state visit-counts?
\item Does $\phi$-EB produce improvement across a range of environments,
or only if rewards are sparse?
\item Can $\phi$-EB with LFA compete with the state-of-the-art in exploration
and deep RL?
\end{itemize}

\subsection{Setup}

We evaluate our algorithm on five games from the Arcade Learning Environment
(ALE), which has recently become a standard high-dimensional benchmark
for RL \cite{BNVB:2013ale}. The reward signal is computed from the
game score. The raw state is a frame of video (a 160$\times$210 array
of 7-bit pixels). There are 18 available actions. The ALE is a particularly
interesting testbed in our context, because the difficulty of exploration
varies greatly between games. Random strategies often work well, and
it is in these games that Deep Q-Networks (DQN) with $\epsilon$-greedy
is able to achieve so-called human-level performance \cite{Mnih2015}.
In others, however, DQN with $\epsilon$-greedy does not improve upon
a random policy, and its inability to explore efficiently is one of
the key determinants of this failure \cite{DBLP:journals/corr/OsbandBPR16}.
We chose five of these games where exploration is hard. Three of the
chosen games have sparse rewards (Montezuma's Revenge, Venture, Freeway)
and two have dense rewards (Frostbite, Q{*}bert).%
\footnote{Note that our experimental evaluation uses the stochastic version
of the ALE \cite{BNVB:2013ale}.%
} 

Evaluating agents in the ALE is computationally demanding. We chose
to focus more resources on Montezuma's Revenge and Venture, for two
reasons: (1) we hypothesise that $\phi$-EB will produce more improvement
in sparse reward games, and (2) leading algorithms with which we seek
to compare $\phi$-EB have also focused on these games. We conducted
five independent learning trials for Montezuma and Venture, and two
trials for the remaining three games. All agents were trained for
100 million frames on the no-op metric \cite{BNVB:2013ale}. Trained
agents were then evaluated for 500 episodes; Table \ref{table} reports
the average evaluation score.

We implement Algorithm \ref{alg:Reinforcement-learning-with} using
Sarsa($\lambda$) with replacing traces and LFA as our RL method,
because it is less likely to diverge than $Q$-learning \cite{Sutton1998}.
To implement LFA in the ALE we use the Blob-PROST feature set presented
in \cite{LMTB:2015shallow}. To date this is the best performing feature
set for LFA in the ALE. The parameters for the Sarsa($\lambda)$ algorithm
are set to the same values as in \cite{LMTB:2015shallow}. Hereafter
we refer to our algorithm as Sarsa-$\phi$-EB. To conduct a controlled
investigation of the effectiveness of $\phi$-EB, we also evaluate
a baseline implementation of Sarsa($\lambda$) with the same features
but with $\epsilon$-greedy exploration (which we denote Sarsa-$\epsilon$).
The same training and evaluation regime is used for both; learning
curves are reported in Figure \ref{fig:Average-training-scores}. 

The $\beta$ coefficient in the $\phi$-exploration bonus was set
to 0.05 for \emph{all games}, after a coarse parameter search. This
search was performed once, across a range of ALE games, and a value
was chosen for which the agent achieved good scores in most games.

\subsection{Results}

\subsubsection*{Comparison with $\epsilon$-greedy Baseline}

\begin{figure*}
\noindent \begin{centering}
\includegraphics[scale=0.65]{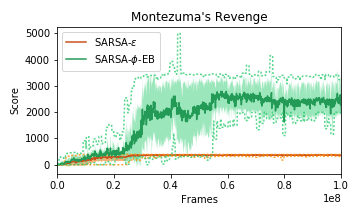}\includegraphics[scale=0.65]{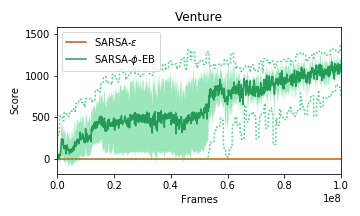}
\par\end{centering}

\noindent \begin{centering}
\includegraphics[scale=0.46]{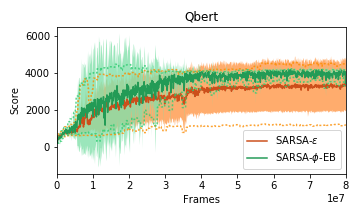}\includegraphics[scale=0.46]{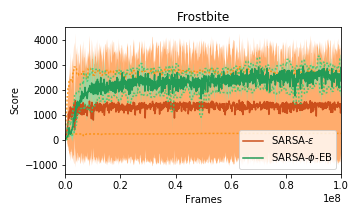}\includegraphics[scale=0.46]{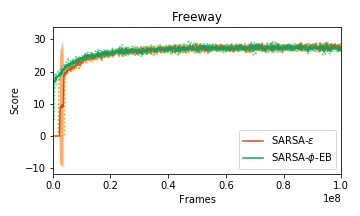}
\par\end{centering}

\caption{Average training scores for Sarsa-$\phi$-EB and the baseline Sarsa-$\epsilon$.
Dashed lines are min/max scores. Shaded regions describe one standard
deviation.\label{fig:Average-training-scores}}
\end{figure*}
\begin{table*}[t]
\begin{centering}
\begin{tabular}{cccccc}
 & \textbf{Venture} & \textbf{Montezuma's Revenge} & \textbf{Freeway} & \textbf{Frostbite} & \textbf{Q{*}bert}\tabularnewline\addlinespace
\midrule
\midrule 
\textbf{\textit{\textcolor{black}{Sarsa-$\phi$-EB}}} & \textit{\textcolor{black}{1169.2}} & \textit{\textcolor{black}{2745.4}} & \textit{\textcolor{black}{0.0}} & \textit{\textcolor{black}{2770.1}} & \textit{\textcolor{black}{4111.8}}\tabularnewline
\midrule 
\textbf{\textcolor{black}{Sarsa-$\epsilon$}} & \textcolor{black}{0.0} & \textcolor{black}{399.5} & \textcolor{black}{29.9} & \textcolor{black}{1394.3} & \textcolor{black}{3895.3}\tabularnewline
\midrule
\midrule 
\textbf{DDQN-PC} & N/A & \textbf{3459} & N/A & N/A & N/A\tabularnewline
\midrule 
\textbf{A3C+} & 0 & 142 & 27 & 507 & 15805\tabularnewline
\midrule 
\textbf{TRPO-Hash} & 445 & 75 & \textbf{34} & \textbf{5214} & N/A\tabularnewline
\midrule 
\textbf{MP-EB} & N/A & 0 & 12 & 380 & N/A\tabularnewline
\midrule
\midrule 
\textbf{DDQN} & 98 & 0 & 33 & 1683 & 15088\tabularnewline
\midrule 
\textbf{DQN-PA} & 1172 & 0 & 33 & 3469 & 5237\tabularnewline
\midrule 
\textbf{Gorila} & \textbf{1245} & 4 & 12 & 605 & 10816\tabularnewline
\midrule 
\textbf{TRPO } & 121 & 0 & 16 & 2869 & 7733\tabularnewline
\midrule 
\textbf{Dueling} & 497 & 0 & 0 & 4672 & \textbf{19220}\tabularnewline
\bottomrule
\end{tabular}
\par\end{centering}

\caption{Average evaluation score for leading algorithms. Sarsa-$\phi$-EB
and Sarsa-$\epsilon$ were evaluated after 100M training frames on
all games except Q{*}bert, for which they trained for 80M frames.
DDQN-PC scores reflect evaluation after 100M training frames. The
MP-EB agent was only trained for 20M frames. All other algorithms
were evaluated after 200M frames.\label{table} Leading scores are
highlighted in bold.}
\end{table*}
 In Montezuma's Revenge, Sarsa-$\epsilon$ rarely leaves the first
room. Its policy converges after an average of 20 million frames.
Sarsa-$\phi$-EB continues to improve throughout training, visiting
up to 14 rooms. The largest improvement over the baseline occurs in
Venture. Sarsa-$\epsilon$ fails to score, while Sarsa-$\phi$-EB
continues to improve throughout training. In Q{*}bert and Frostbite,
the difference is less dramatic. These games have dense, well-shaped
rewards that guide the agent's path through state space and elide
$\epsilon$-greedy's inefficiency. Nonetheless, Sarsa-$\phi$-EB consistently
outperforms Sarsa-$\epsilon$ throughout training so its cumulative
reward is much higher. 

In Freeway, Sarsa-$\phi$-EB with $\beta=0.05$ fails to match the
performance of the baseline algorithm, but with $\beta=0.035$ it
performs better (Figure \ref{fig:Average-training-scores} shows the
learning curve for the latter). This sensitivity to the $\beta$ parameter
likely results from the large number of unique Blob-PROST features
that are active in Freeway, many of which are not relevant for finding
the optimal policy. If $\beta$ is too high the agent is content to
stand still and receive exploration bonuses for observing new configurations
of traffic. This accords with our hypothesis that efficient optimistic
exploration should involve measuring novelty with respect to task-relevant
features.

In summary, Sarsa-$\phi$-EB with $\beta=0.05$ outperforms Sarsa-$\epsilon$
on all tested games except Freeway. Since both use the same feature
set and RL algorithm, and differ only in their exploration policies,
this is strong evidence that $\phi$-EB produces improvement over
random exploration across a range of environments. This also supports
our conjecture that using the same features for value function approximation
and novelty estimation is an appropriate way to generalise visit-counts
to the high-dimensional setting.

\subsubsection*{Comparison with Leading Algorithms}

Table \ref{table} compares our evaluation scores to Double DQN (DDQN)
\cite{HGS:2016doubleQ}, Double DQN with pseudocount (DDQN-PC) \cite{Bellemare2016},
A3C+ \cite{Bellemare2016}, DQN Pop-Art (DQN-PA) \cite{DBLP:journals/corr/HasseltGHS16},
Dueling Network (Dueling) \cite{WFL:2015DQN}, Gorila \cite{DBLP:journals/corr/NairSBAFMPSBPLM15},
DQN with Model Prediction Exploration Bonuses (MP-EB) \cite{DBLP:journals/corr/StadieLA15},
Trust Region Policy Optimisation (TRPO) \cite{DBLP:journals/corr/SchulmanLMJA15},
and TRPO-AE-SimHash (TRPO-Hash) \cite{Tang2016exploration}. The most
interesting comparisons for our purposes are with TRPO-Hash, DDQN-PC,
A3C+, and MP-EB, because these algorithms all use exploration strategies
that drive the agent to reduce its uncertainty. TRPO-Hash, DDQN-PC,
and A3C+ are count-based methods, MP-EB seeks high model prediction
error.

Our Sarsa-$\phi$-EB algorithm achieves an average score of 2745.4
on Montezuma: the second highest reported score. On this game it far
outperforms every algorithm apart from DDQN-PC, despite only having
trained for half the number of frames. Note that neither A3C+ nor
TRPO-Hash achieves more than 200 points, despite their exploration
strategies. 

On Venture Sarsa-$\phi$-EB also achieves state-of-the-art performance.
It achieves the third highest reported score despite its short training
regime, and far outperforms A3C+ and TRPO-Hash. DDQN-PC evaluation
scores are not given for Venture, but reported learning curves suggest
Sarsa-$\phi$-EB performs much better here \cite{Bellemare2016}.
The performance of Sarsa-$\phi$-EB in Frostbite also seems competitive
given the shorter training regime. Nonlinear algorithms perform better
in Q{*}bert. In Freeway Sarsa-$\phi$-EB fails to score any points,
for reasons already discussed.

\section{Conclusion}

We have introduced the $\phi$-Exploration Bonus method, a count-based
optimistic exploration strategy that scales to high-dimensional environments.
It is simpler to implement and less computationally demanding than
some other proposals. Our evaluation shows that it improves upon $\epsilon$-greedy
exploration on a variety of games, and that it is even competitive
with leading exploration techniques developed for deep RL. Unlike
other methods, it does not require the design of an exploration-specific
state representation, but rather exploits the features used in the
approximate value function. We have argued that computing novelty
with respect to these task-relevant features is an efficient and principled
way to generalise visit-counts for exploration. We conclude by noting
that this reliance on the feature representation used for LFA is also
a limitation. It is not obvious how a method like ours could be combined
with the nonlinear function approximation techniques that have driven
recent progress in RL. We hope the success of our simple method will
inspire future work in this direction.

\bibliographystyle{named}
\bibliography{bibliographyNEWNoURL}

\end{document}